\PassOptionsToPackage{ruled,vlined,linesnumbered}{algorithm2e}
\documentclass{arxiv}
\usepackage{algorithm2e}
\SetNlSty{}{}{} 
\usepackage{float}
\usepackage{graphicx} 
\usepackage{amsmath,amsfonts}
\usepackage{amssymb}
\usepackage{color} 
\usepackage{array}
\usepackage{ragged2e}

\usepackage{subcaption}
\usepackage{textcomp}
\usepackage{url}
\usepackage{verbatim}
\usepackage{graphicx}
\usepackage{epstopdf}
\usepackage{dcolumn}
\usepackage{balance}
\usepackage{mathtools}
\usepackage{multirow}
\usepackage{colortbl}
\usepackage{adjustbox}
\usepackage[short]{optidef}
\usepackage{wrapfig}

\usepackage{mathptmx} 

\newtheorem{mydefinition}{\bf{Definition}}
\newtheorem{assumption}{\bf{Assumption}}

\newtheorem{problem}{\bf{Problem}}

\newtheorem{mytheorem}{\bf{Theorem}}

\newtheorem{mycorollary}{Corollary}[mytheorem]

\usepackage{hyperref}
\newcommand{\funcstyle}[1]{{\protect\NoHyper\FuncSty{#1}\protect\endNoHyper}}

\DeclareMathOperator*{\argmax}{arg\,max}
\DeclareMathOperator*{\argmin}{arg\,min}

\usepackage{tikz}
\newcommand*\circled[1]{\tikz[baseline=(char.base)]{
            \node[shape=circle,draw,inner sep=0.1pt] (char) {#1};}}
\usepackage{enumitem}

\title[Probabilistic Satisfaction of Temporal Logic Constraints in Reinforcement Learning]{Probabilistic Satisfaction of Temporal Logic Constraints in Reinforcement Learning via Adaptive Policy-Switching}

\author{%
 \Name{Xiaoshan Lin} \Email{lin00668@umn.edu}\\
 \addr University of Minnesota, Minnesota, USA
\AND
 \Name{Sad{\i}k Bera Y\"{u}ksel} \Email{yuksel.sa@northeastern.edu}\\
  \Name{Yasin Yaz{\i}c{\i}o\u{g}lu} \Email{y.yazicioglu@northeastern.edu}\\
  \Name{Derya Aksaray} \Email{d.aksaray@northeastern.edu}\\
 \addr Northeastern University, Massachusetts, USA
}

\begin{document}

\maketitle

\begin{abstract}
Constrained Reinforcement Learning (CRL) is a subset of machine learning that introduces constraints into the traditional reinforcement learning (RL) framework. Unlike conventional RL which aims solely to maximize cumulative rewards, CRL incorporates additional constraints that represent specific mission requirements or limitations that the agent must comply with during the learning process. In this paper, we address a type of CRL problem where an agent aims to learn the optimal policy to maximize reward while ensuring a desired level of temporal logic constraint satisfaction throughout the learning process. We propose a novel framework that relies on switching between pure learning (reward maximization) and constraint satisfaction. This framework estimates the probability of constraint satisfaction based on earlier trials and properly adjusts the probability of switching between learning and constraint satisfaction policies. We theoretically validate the correctness of the proposed algorithm and demonstrate its performance through comprehensive simulations. 
\end{abstract}


\section{Introduction}

Reinforcement learning (RL) relies on learning optimal policies through trial-and-error interactions with the environment. However, many real-life systems need to not only maximize some objective function but also satisfy certain constraints on the system's trajectory. Conventional formulations of constrained RL (e.g. \cite{achiam2017constrained,chow2018lyapunov,garcia2015comprehensive}) focus on maximizing reward functions while keeping some cost function below a certain threshold. In contrast, robotic systems often require adherence to more intricate spatial-temporal constraints. For instance, a robot should ``pick up from region A and deliver to region B within a specific time window, while avoiding collisions with any object".

 Temporal logic (TL) is a formal language that can express spatial and temporal specifications. In recent years, RL subject to TL constraints has gained significant interest, especially in the robotics community. One common approach involves encoding constraint satisfaction into the reward function and learning a policy by maximizing the cumulative reward (e.g., \cite{li2017reinforcement, aksaray2016q}). Another approach focuses on modifying the exploration process during RL, such as the shielded RL proposed in \cite{alshiekh2018safe} that corrects unsafe actions to satisfy Linear Temporal Logic (LTL) constraints. Similarly, \cite{hasan2020} constructs a safe padding based on maximum likelihood estimation and Bellman update, combined with a state-adaptive reward function, to maximize the probability of satisfying LTL constraints. A model-based approach is introduced for safe exploration in deep RL by \cite{cai2021safe}, which employs Gaussian process estimation and control barrier functions to ensure a high likelihood of satisfying LTL constraints. Although these approaches focus on maximizing the probability of satisfaction, they do not provide guarantees on satisfying TL constraints with a \emph{desired probability} during the learning process. Moreover, \cite{jansen2020safe} proposes a method based on probabilistic shield and model checking to ensure the satisfaction of LTL specifications under model uncertainty. However, this method lacks guarantees during the early stages of the learning process. Finally, \cite{aksaray2021probabilistically} and \cite{lin2023timewindow} assume partial knowledge about the system model and leverage it to prune unsafe actions, thus ensuring the satisfaction of Bounded Temporal Logic (BTL) with a desired probabilistic guarantee throughout the learning process. However, these two methods require learning over large state spaces which leads to scalability issues. Furthermore, \cite{aksaray2021probabilistically}, which is the closest work to this paper, is only applicable to a more restrictive family of BTL formulas.
    
Driven by the need for a computationally efficient solution that offers desired probabilistic constraint satisfaction guarantees throughout the learning process (even in the first episode of learning), we propose a novel approach that enables the RL agent to alternate between two policies during the learning process. The first policy is a stationary policy that prioritizes satisfying the BTL constraint, while the other employs RL to learn a policy on the MDP that only maximizes the cumulative reward. The proposed algorithm estimates the satisfaction rate of following the first policy and adaptively updates the switching probability to balance the need for constraint satisfaction and reward maximization. We theoretically show that, with high confidence, the proposed approach satisfies the BTL constraint with a probability greater than the desired threshold. We also validate our approach via simulations.

\section{Preliminaries: Bounded Temporal Logic}
Bounded temporal logics (BTL) (e.g., Bounded Linear Temporal Logic \cite{zuliani2010bayesian}, Interval Temporal Logic \cite{cau1997refining}, and Time Window Temporal Logic (TWTL) \cite{twtl}) are expressive languages that enable users to define specifications with explicit time-bounds (e.g., ``visit region A and then region B within a desired time interval").   
We denote the set of positive integers by $\mathbb{Z}^{+}$, the set of atomic propositions by $AP$, and the power set of a finite set $\Sigma$ by $2^\Sigma$. In this paper, we focus on BTL that can be translated into a finite-state automaton.


\begin{mydefinition} (Finite State Automaton) A finite state automaton (FSA) is a tuple $\mathcal{A} = (Q, q_{init}, 2^\Sigma, \delta, F)$, where $Q$ is a finite set of states, $q_{init}$ is the initial state, $2^\Sigma$ is the input alphabet, $\delta : Q \times 2^\Sigma \rightarrow Q$ is a transition function, and $F$ is the set of accepting states.
\end{mydefinition}


While our proposed methods can be applied to any BTL that can be translated into an FSA, we will use TWTL specifications in our examples. Hence, we also provide some relevant preliminaries here. A TWTL \cite{twtl} formula over a set of atomic propositions $\Sigma$ is defined as follows:
    \begin{equation*}
        \phi ::= H^ds\,\,|\,\, H^d\neg s \,\,|\,\, \phi_1 \land \phi_2 \,\,|\,\, \phi_1 \lor \phi_2 \,\,|\,\, \neg\phi_1 \,\,|\,\, \phi_1 \cdot \phi_2 \,\,|\,\, [\phi_1]^{[a,b]}.
    \end{equation*}
Here, $s$ either represents the constant ``true" or an atomic proposition in $AP$;  $\phi_1$ and $\phi_2$ are TWTL formulas; $\land$, $\lor$, and $\neg$ denote the conjunction, disjunction, and negation Boolean operators, respectively; $\cdot$ is the concatenation operator; operator $H^d$ where $d \in \mathbb{Z}^{+}$, represents the hold operator; $[]^{[a,b]}$ denotes the within operator with $a, b \in \mathbb{Z}^{+}$ and $a \leq b$. For example, the statement ``stop at location A for 3 seconds" can be represented as $H^3\text{A}$, and ``take the customer to A within 20 minutes, and then pick up food from B within 60 minutes” can be written as $[H^0 A]^{[0,20]}\cdot[H^0 B]^{[0,60]}$. Detailed syntax and semantics of TWTL can be found in \cite{twtl}.

In this paper, we also allow temporal relaxations of TWTL specifications that can be encoded into a compact FSA representation. Temporally relaxed TWTL formulas accommodate tasks that may be completed ahead of or after their original deadlines. In that case, we can capture violation cases without the need of a total FSA. For instance, a formula $\phi = [H^0 A]^{[0,20]}\cdot[H^0 B]^{[0,60]}$ can be temporally relaxed as $\phi(\tau)=[H^0 A]^{[0,20 + \tau_1]}\cdot[H^0 B]^{[0,60 + \tau_2]}$, where $\tau = (\tau_1, \tau_2)$. Specifically, we consider a relaxed formula $\phi(\tau)$ whose time bound $\|\phi(\tau)\|$ does not exceed the time bound of $\phi$ (i.e., any delay in achieving tasks need to be compensated by the others to ensure the overall mission duration is not exceeded). 
 (Note: The time bound of $\phi$ is defined as the maximum time needed to satisfy $\phi$.)

\section{Problem Statement}
We consider a labeled-Markov Decision Process (MDP) denoted as $\mathcal{M}=(S,A,\Delta_M,R, l)$, where $S$ represents the state space, and $A$ denotes the set of actions. The probabilistic transition function is defined as $\Delta_M: S \times A \times S \rightarrow [0,1]$, while $R: S \rightarrow \mathbb{R}$ represents the reward function. Additionally, $l: S \rightarrow 2^{AP}$ is a labeling function that maps each state to a set of atomic propositions. An example MDP is shown in Fig.~\ref{fig:mdp}. Given a trajectory $\mathbf{s} = s_1s_2\,...$ over the MDP, the output word $\mathbf{o} = o_1o_2\,...$ is a sequence of elements from $2^{AP}$, where each element $o_i=l(s_i)$. The subword $o_i\,...o_j$ is denoted by $
\mathbf{o}_{i,j}$

\setlength{\belowcaptionskip}{-10pt}
\begin{figure}[h!]
    \begin{center}
        \resizebox*{0.75\columnwidth}{!}{\includegraphics{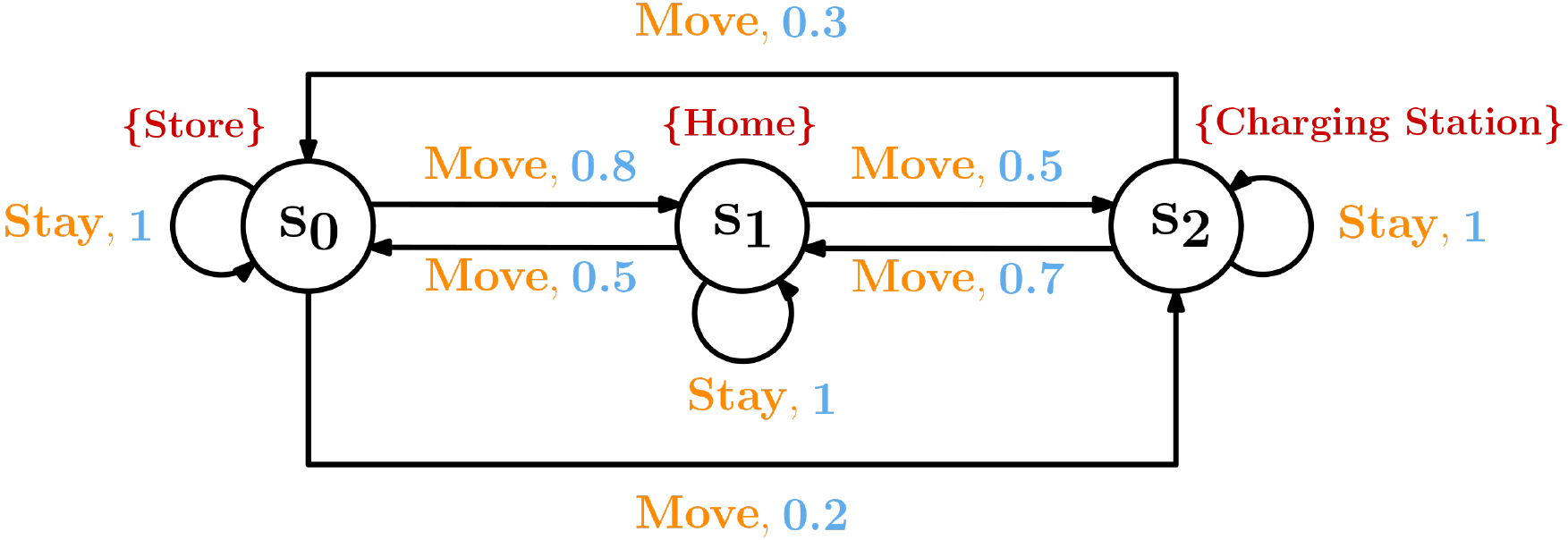}}%
        \caption{An MDP where $S=\{s_0,s_1, s_2\}$, $A=\{Move,Stay\}$, $AP=\{Home,Store, Charging\,\,Station\}$, $l(s_0)=\{Store\}$, $l(s_1) =\{Home\}$, $l(s_2)=\{Charging\,\,Station\}$. Edge labels indicate the corresponding action and transition probability.}
        \label{fig:mdp}
    \end{center}
\end{figure}

\begin{mydefinition} [Deterministic Policy] Given a labeled-MDP $\mathcal{M}=(S,A,\Delta_M,R, l)$, a deterministic policy is a mapping $\pi: S \rightarrow A$ that maps each state to a single action.
\end{mydefinition}

We address the problem of learning a policy that maximizes the reward while ensuring the satisfaction of a BTL specification with a probability greater than a desired threshold \emph{throughout the learning process}. Accordingly, while the policy improves to collect more reward, the desired probabilistic constraint satisfaction is guaranteed even in the first episode of learning. 
Clearly, such a formal guarantee cannot be achieved without any prior knowledge about the transition probabilities. In this paper, we assume that while the actual transition probabilities may be unknown, for each state $s$ and action $a$, the agent knows which states have a non-zero probability and which states have a sufficiently large probability of being observed as the next state $s'$. 

\begin{problem}    \label{problem_generic}     
    Suppose that the following are given: 
        \begin{itemize} 
        \item a labeled-MDP $\mathcal{M}=(S,A,\Delta_M,R, l)$ with unknown transition function $\Delta_M$ and reward function $R$,
        \item a BTL formula $\phi$ with time bound $\|\phi\|=T$,
        \item a desired probability threshold ${Pr_{des} \in (0,1]}$,
        \item some $\epsilon \in [0,1)$ such that for each MDP state $s$ and action $a$, the states $s'$ for which $\Delta_M(s,a,s')> 0$  and the states $s''$ for which $\Delta_M(s,a,s'') \geq 1-\epsilon$ are known, 
        \end{itemize}  find the optimal policy
        \begin{equation}
           \label{problem_objective}
            \pi^* = \arg\max\limits_{\pi }E^{\pi}[\sum_{t=0}^{\infty}\gamma^t r_t]
        \end{equation}
        such that, for every episode j in the learning process,
        \begin{equation}\label{problem_constraint}
        \begin{aligned} 
            \Pr\big(\mathbf{o}_{jT,jT + T} \models \phi{(\tau_j)}\big) \color{black} & \geq Pr_{des}, \quad \forall j\geq0 \\
            \|\phi(\tau_j) \| & \leq T
        \end{aligned}
        \end{equation}
        \noindent where $\gamma$ is a discount factor, $\mathbf{o}_{jT,jT + T}$ is the output word in episode j, $\tau_j$ is the time relaxation in episode $j$, $\phi(\tau_j)$ is a temporally-relaxed BTL constraint, and $\|\phi(\tau_j) \|$ is the time bound of $\phi(\tau_j)$. 
    \end{problem}

\section{Proposed Algorithm} 
We propose a solution to Problem \ref{problem_generic} by introducing a switching-based algorithm that allows switching between two policies: 1) a stationary policy derived from the product of the MDP and FSA for maximizing the probability of constraint satisfaction based on the available prior information, and 2) a policy learned over the MDP to maximize rewards. Before each episode, the RL agent determines whether to follow the constraint satisfaction policy or the reward maximization policy based on a computed switching probability. The proposed approach, separating constraint satisfaction from reward maximization, eliminates the need for a time-product MDP often used in the state-of-the-art, offering a more computationally efficient solution.  

\subsection{Policy for Constraint satisfaction}
Consider a task ``eventually visit A and then B". Suppose that the agent is at C. The agent must select an action that steers it towards 1) B if A has been visited before; or 2) A if A has not been visited yet. Hence, the selection of actions is determined by the agent's current state and the progress of constraint satisfaction, which can be encoded by a Product MDP.

\begin{mydefinition} [Product MDP]
        \label{def:aug-pa}
        Given a labeled-MDP $\mathcal{M} = (S,A,\Delta_M,R, l)$ and an FSA $\mathcal{A} = (Q, q_{init}, $ $O,\delta, F)$, a product MDP is a tuple $\mathcal{P} =\mathcal{M} \times \mathcal{A} = (S_{P},S_{P,init},A,\Delta_{P}, R_{P}, F_{P})$, where
        \begin{itemize}
            \item $S_{P} = S \times Q$ is a finite set of states;
           \item $S_{P,init}=\{(s,\delta(q_{init}, l(s))\, | \,\forall s \in S\}$ is the set of initial states, where $\delta$ is the transition function of the FSA;
            \item $A$ is the set of actions;
            \item $\Delta_{P} : S_{P} \times A \times S_{P} \rightarrow [0,1]$ is the probabilistic transition relation such that for any two states, $p=(s,q) \in S_{P}$ and $p^{\prime}=(s^{\prime},q^{\prime}) \in S_{P}$, and any action $a \in A$, $\Delta_{P}(p,a,p^\prime)=\Delta_{M}(s,a,s^\prime)$ if $\delta(q,l(s))=q^\prime$; 
            \item $R_{P}:S_{P} \rightarrow \mathbb{R}$ is the reward function such that $R_{P}(p) = R(s)$ for $p=(s,q) \in S_{P}$;
            \item $F_{P} = (S \times F_A) \subseteq S_{P}$ is the set of accepting states.
        \end{itemize}

    \end{mydefinition}  

The policy for constraint satisfaction is designed to maximize the probability of reaching $F_{P}$ from any state of the product MDP. We can obtain such a policy by selecting the action to minimize the expected distance to $F_{P}$ from each state. Thus, we define $\epsilon$-stochastic transitions and distance-to-$F_{P}$.

\begin{mydefinition}[$\epsilon$-Stochastic Transitions]
Given a product MDP and some $\epsilon \in [0,1)$, any transition $(p_i,a,p_j)$ such that $\Delta_{P}(p_i,a,p_j) \geq 1-\epsilon$ is defined as an $\epsilon$-stochastic transition.   
\end{mydefinition}

If transition $(p_i,a,p_j)$ is a $0$-stochastic transition, the agent will move to state $p_j$ with probability 1 after taking action $a$ at state $p_i$. Oppositely, as $\epsilon$ approaches 1, any feasible transition becomes a $\epsilon$-stochastic transition. Next, we will use $\epsilon$-stochastic transitions to define \textit{Distance-To-}$F_{P}$.  

\begin{mydefinition}
     [Distance-To-$F_{P}$]
    Given a product MDP, for any product MDP state $p$, the distance from $p$ to the set of accepting states $F_{P}$ is 
    \begin{equation}\label{dist-to-F}
        D^\epsilon(p) = \min\limits_{p' \in F_{P}} dist^\epsilon(p,p')
    \end{equation} 
    where $dist^\epsilon(p,p')$ represents the minimum number of $\epsilon$-stochastic transitions to move from state $p$ to state $p'$. 
\end{mydefinition}
    
The distance-to-$F_{P}$, $D^\epsilon(p)$, represents the minimum number of $\epsilon$-stochastic transitions needed for reaching the set of accepting states from a state $p$. We will use this metric to design a policy for constraint satisfaction (reaching the accepting states) and derive a lower bound on the probability of constraint satisfaction within a time bound.

\begin{mydefinition}
[$\pi_{GO}^{\epsilon}$ Policy] Given a product MDP and $\epsilon \in [0,1)$, $\pi^{\epsilon}_{GO} : S_{P} \rightarrow A$, is a stationary policy over the product MDP such that
\begin{equation}\label{eq:pigo}
   \pi^{\epsilon}_{GO}(p) = \argmin \limits_{a\in A} D^\epsilon_{min}(p,a),
\end{equation}
\text{where} $D^\epsilon_{min}(p,a) = \min \limits_{p':\Delta_{P}(p,\,a,\,p')\geq 1-\epsilon} D^{\epsilon}(p')$, i.e., the minimum distance-to-$F_P$ among the states reachable from $p$ under action $a$ with probability of at least $1-\epsilon$.
\end{mydefinition}

The policy $\pi^{\epsilon}_{GO}$ aims to reduce the distance-to-$F_{P}$ and is computed in Alg.~\ref{alg:offline}. 
The inputs of Alg.~\ref{alg:offline} are the transition uncertainty, the MDP, and the BTL constraint.
First, an FSA is generated from the BTL formula, and a product MDP is constructed using this FSA and the given MDP (lines 1-2). Then the product MDP is used to calculate the distance-to-$F_{P}$ for all states (line 3). Finally, the $\pi_{GO}^{\epsilon}$ policy is computed by selecting the action that minimizes the distance-to-$F_{P}$ for each state (lines 4-5). 

 \begin{algorithm2e}[htb!]
\SetKwInOut{Input}{Input}
\SetKwInOut{Output}{Output}
\Input{\justifying{uncertainty $\epsilon \in [0,1)$;\,MDP $\mathcal{M} = (S,A,\Delta_M,R, l)$;\, BTL formula $\Phi$ }}
\Output{$\pi_{GO}^{\epsilon}(\cdot)$ policy}
\caption{Off-line computation of $\pi^{\epsilon}_{GO}$ policy}\label{alg:offline}
\DontPrintSemicolon
 Create FSA of $\Phi$, $\mathcal{A} = (Q, q_{init}, 2^{AP}, \delta, F_A )$;\\
 Create product MDP, $\mathcal{P} = \mathcal{M} \times \mathcal{A} = (S_{P},P_{init},A,\Delta_{P}, R_{P}, F_{P})$;\\
Calculate the distance-to-$F_{P}$, i.e., $d^{\epsilon}(p)$ for all $p \in S_{P}$;\\
\For{each $p \in S_{P}$}{
 $\pi_{GO}^{\epsilon}(p) \leftarrow (\ref{eq:pigo})$
}
\end{algorithm2e}   

  In order to achieve probabilistic constraint satisfaction in each episode, our approach builds a conservative estimation of the probability of reaching an accepting state under the $\pi_{GO}^{\epsilon}$ policy for every initial state. In particular, we present two methods for computing a lower bound on this satisfaction probability. \emph{The first method} involves a closed-form equation outlined in Theorem \ref{theorem1}. To derive this equation, we introduce an additional assumption on the product-MDP. \emph{The second method}, which relaxes this assumption, utilizes a recursive approach. 

   \begin{assumption} \label{assumption_1}
    The MDP and FSA (constraint) are such that any feasible transition in the resultant product MDP cannot increase the distance-to-$F_{P}$ by more than some $\delta_{max}\in \mathbb{Z^+}$.
     \end{assumption}
   \noindent Assumption \ref{assumption_1} is the relaxed version of the assumption made in \cite{aksaray2021probabilistically}, which only allows $\delta_{max}=1$.  This relaxed assumption poses fewer constraints on the BTL formulas that can be handled by our proposed approach. For example, we can accommodate formulas such as ``eventually stay at A for 3 time steps". The distance-to-$F_{P}$ can increase by 3 if the robot leaves A right before it stays at A for 3 time steps.
    
\begin{mytheorem} \label{theorem1} Let Assumption \ref{assumption_1} hold. For any $p \in S_{P}$ of the given product MDP $\mathcal{P} = (S_{P},P_{init},A,\Delta_{P}, R_{P}, F_{P})$, let integer $k > 0$ denote the remaining time steps, $d= D^\epsilon(p)$ denote the distance-to-$F_{P}$ from $p$, and $Pr(p \xrightarrow{k} F_{P}; \pi_{GO}^{\epsilon})$ be the probability of reaching $F_{P}$ from $p$ within the next $k$ time steps under the policy $\pi_{GO}^{\epsilon}$.  If $0 < d < \infty$, then
    \begin{equation}\label{lower_bound}
       Pr(p\xrightarrow{k} F_{P}; \pi_{GO}^{\epsilon}) \geq lb^c[p][k],
    \end{equation}
    where
    
    \begin{equation*}
        \begin{split}
            &lb^c[p][k] = \sum_{m=1}^{k} P(T_{m}), \\
            &P(T_{m}) = \left[C_{m}^{\frac{m-d}{1+\delta_{max}}}\epsilon^{\frac{m-d}{1+\delta_{max}}}(1-\epsilon)^{\frac{m\delta{max}+d}{1+\delta_{max}}} - \sum_{m'=1}^{m-1}C_{m-m'}^{\frac{m-m'}{1+\delta_{max}}}\epsilon^{\frac{m-m'}{1+\delta_{max}}}(1-\epsilon)^{\frac{(m-m')\delta_{max}}{1+\delta_{max}}}P(T_{m'})\right], \\
            &C_{m}^{n} =
         \begin{cases}
          0 & \text {if }n > m \hspace{0.4em} \text{or $n \notin \mathbb{Z^+}$,}  \\
           \frac{m!}{n!(m-n)!} & \text{otherwise.}
          \end{cases} 
        \end{split}
    \end{equation*} 
    \normalsize
 \end{mytheorem}

 \begin{proof}
     See the appendix.
 \end{proof}

\begin{mycorollary}\label{corollary_1} 
 Let Assumption \ref{assumption_1} hold. For any initial state $p_0 \in P_{init}$ such that $0<D^\epsilon(p_0)<\infty$ and $lb^c[p_0][T] \geq Pr_{des}$, 
 \begin{equation}\label{coreq}
       Pr(p_0\xrightarrow{T} F_{P}; \pi_{GO}^{\epsilon}) \geq Pr_{des},
    \end{equation}
    where $T$ is the time bound of the BTL constraint.
 \end{mycorollary}

\begin{proof}
For any $p_0 \in P_{init}$ satisfying $0<D^\epsilon(p_0)<\infty$, by plugging $k=T$ into \eqref{lower_bound}, we obtain $Pr(p_0\xrightarrow{T} F_{P}; \pi_{GO}^{\epsilon}) \geq lb^c[p_0][T]$, which implies \eqref{coreq} when $lb^c[p_0][T] \geq Pr_{des}$.
\end{proof}
\hspace{4pt} 

While the lower bound $lb^c$ can be efficiently computed, it can be overly conservative due to the assumptions made in deriving \eqref{lower_bound}. To address this limitation, we employ a method proposed in \cite{lin2023timewindow} and present Alg.~\ref{alg:recursive}, which computes another lower bound based on recursive computations over the product MDP. While this recursive approach is computationally more demanding, it provides a much less conservative lower bound than \eqref{lower_bound}.

\begin{algorithm2e}[h]
    \SetKwInOut{Input}{Input}
    \SetKwInOut{Output}{Output}
    \Input{\justifying{product MDP $\mathcal{P} = (S_{P},P_{init},A,\Delta_{P}, R_{P}, F_{P})$;\, time bound of $\phi$, i.e., $T$}}
\Output{lower bound of satisfaction probability\textit{$lb^r$[$\cdot$][$\cdot$]}}
\caption{Recursive algorithm for computing the proposed lower bound in \eqref{eq:opt_problem} }\label{alg:recursive}
\DontPrintSemicolon
  \For{k = 0, 1, ..., T}{
     \For{each $p \in S_P$}{
          \If{k = 0} {$lb^r[k][p] \leftarrow 1 $ \textbf{if} p is accepting state, \textbf{else} $lb^r[k][p]\leftarrow 0$}
          \lElseIf{$D^\epsilon(p) > k$} { $lb^r[k][p] \leftarrow 0$}
          \lElseIf{p is accepting state} { $lb^r[k][p] \leftarrow 1$} 
          \Else{
          $lb^r[k][p] \leftarrow $ solve \eqref{eq:opt_problem}
          } 
        }
  }
\end{algorithm2e}
\hspace{-2pt}
\begin{mini!}|s|[2]
    {\Delta_{1},\Delta_{2},\hdots,\Delta_{n}}                               
    {\sum_{i=1}^{n} lb^r[p'_i][k-1] \Delta_{i} \label{eq:eq1}}   
    {\label{eq:opt_problem}} 
    {lb^r[p][k] = }                       
    \addConstraint{\sum_{i=1}^{n} \Delta_{i}}{=1 \label{eq:con1}}    
    \addConstraint{ 1-\epsilon \leq \Delta_{j} }{\leq 1, j = 1, 2, \hdots, m \label{eq:con2}}
     \addConstraint{ 0 < \Delta_{k} }{\leq 1, k = m+1,\hdots,n. \label{eq:con3}}
\end{mini!}


The inputs to Alg.~\ref{alg:recursive} include the product automaton and the time bound of the BTL constraint. It outputs $lb^r[p][k]$, the lower bound of the satisfaction probability from any product automaton state $p$ within $k$ time steps, under policy $\pi^\epsilon_{GO}$. The algorithm is based on the fact that the lower bound $lb^r[p][k]$ depends on $lb^r[p'][k-1]$, where $p'$ is the reachable states from $p$ under $\pi^\epsilon_{GO}$. If the values of $lb^r[p'][k-1]$ are known, we can compute $lb^r[p][k]$, by solving the optimization problem in \eqref{eq:opt_problem}.  Starting from $k=0$ (lines 3-4), we can iteratively solve the optimization problem for $lb^r[p][k]$ up to $k=T$ (lines 5-8).

\begin{figure}[htb!]
    \begin{center}
    \resizebox*{0.5\textwidth}{!}{\includegraphics{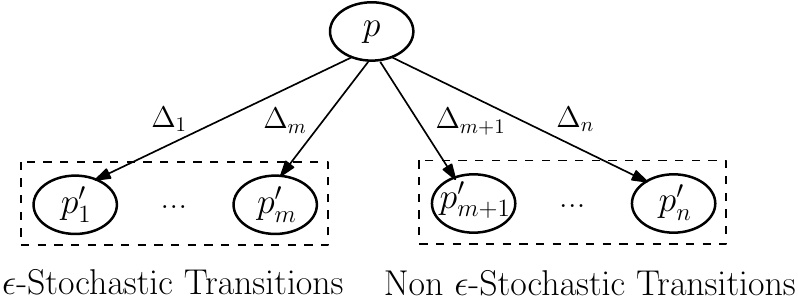}}%
    \caption{Possible transitions from state $p$ under $\pi^\epsilon_{GO}$, where $\Delta_i$ denotes the unknown transition probabilities.}
    \label{fig:opt}  
    \end{center}
    \end{figure}



\setlength{\belowcaptionskip}{-15pt}
\begin{mytheorem}
    \label{proposition2}
For any $p \in S_{P}$ of a given product MDP $\mathcal{P} = (S_{P},P_{init},A,\Delta_{P}, R_{P}, F_{P})$, let integer $k > 0$ denote the remaining time steps, and $Pr(p \xrightarrow{k} F_{P}; \pi_{GO}^{\epsilon})$ be the probability of reaching the set of accepting states from $p$ within the next $k$ time steps under the policy $\pi_{GO}^{\epsilon}$. Then
\begin{equation}\label{lower_bound_r}
       Pr(p\xrightarrow{k} F_{P}; \pi_{GO}^{\epsilon}) \geq lb^r[p][k].
    \end{equation}
\end{mytheorem}

\begin{proof}
This result is a special case of Lemma 1 in \cite{lin2023timewindow}, where we substitute the constraint $\Delta_{min}(p_i^{t},a,p^{t+1}_j) \leq \hat{\Delta}_{j} \leq \Delta_{max}(p_i^{t},a,p^{t+1}_j)$ ((8) in \cite{lin2023timewindow}) with constraints \eqref{eq:con2}, \eqref{eq:con3}.
\end{proof}

 \begin{mycorollary}\label{corollary_2}
 Given a product MDP $\mathcal{P} = (S_{P},P_{init},A,\Delta_{P}, R_{P}, F_{P})$,  any initial state $p_0 \in P_{init}$ such that  $lb^r[p_0][T] \geq Pr_{des}$ satisfies  
 \begin{equation}\label{coreq2}
       Pr(p_0\xrightarrow{T} F_{P}; \pi_{GO}^{\epsilon}) \geq Pr_{des}.
    \end{equation}
 \end{mycorollary}

 \begin{proof}
For any $p_0 \in P_{init}$, by plugging $k=T$ into \eqref{lower_bound_r}, we obtain $Pr(p_0\xrightarrow{T} F_{P}; \pi_{GO}^{\epsilon}) \geq lb^r[p_0][T]$, which implies \eqref{coreq2} when $lb^r[p_0][T] \geq Pr_{des}$.
 \end{proof}

\subsection{A Switching-based RL Algorithm}

The switching algorithm allows a probabilistic transition between two distinct policies: the $\pi_{GO}^{\epsilon}$ policy for constraint satisfaction and a learned policy for reward maximization. We hereby define the switching policy as follows.

 \begin{mydefinition}[Switching Policy] For any episode starting with initial state $p \in P_{init}$, the switching policy is defined as adopting $\pi_{GO}^{\epsilon}$ policy with a probability of $Pr_{switch}(p)$ and RL with a probability of $1-Pr_{switch}(p)$ throughout that episode, where $Pr_{switch}(p)$ is the switching probability for state $p$.
 \end{mydefinition}

 To determine the switching probability $Pr_{switch}(p)$ for each initial state $p$, we initialize $Pr_{switch}(p)$ to 1 so that the agent adopts $\pi_{GO}^{\epsilon}$ policy with probability 1 in the early stage of the learning process. Let $Pr_{GO}(p)$ denote the probability of satisfaction under policy $\pi_{GO}^{\epsilon}$ starting from an initial state $p$. By estimating $Pr_{GO}(p)$, we can adjust the switching probability such that: 1) 
$Pr_{switch}(p)$ is lower than 1 (to allow exploration for reward maximization) if we are confident that $Pr_{GO}(p)$ is greater than the desired threshold $Pr_{des}$; 2) $Pr_{switch}(p)$ remains 1 (to maximize constraint satisfaction) if we are not confident that $Pr_{GO}(p)$ is greater than $Pr_{des}$. 

Since $\pi_{GO}^{\epsilon}$ is a stationary policy, for any initial state $p$, the outcome of following $\pi_{GO}^{\epsilon}$ (either satisfies or violates the constraint) is a Bernoulli trial with the probability of success (constraint satisfaction) equal to $Pr_{GO}(p)$.  Accordingly, we use Wilson score interval \cite{wilson1927probable}, to compute a confidence bound [$Pr_{low}(p)$, $Pr_{up}(p)$] that contains $Pr_{GO}(p)$ up to some given confidence level, where  

\begin{equation} \label{upper} Pr_{up}(p) =  \frac{n_S(p)+\frac{1}{2}z^2}{n(p)+z^2} + \frac{z}{n(p)+z^2}\sqrt{\frac{n_S(p) n_F(p)}{n(p)}+\frac{z^2}{4}},
\end{equation}

 \begin{equation}\label{lower} 
 Pr_{low}(p) =  \frac{n_S(p)+\frac{1}{2}z^2}{n(p)+z^2} - \frac{z}{n(p)+z^2}\sqrt{\frac{n_S(p) n_F(p)}{n(p)}+\frac{z^2}{4}}.
\end{equation}


Here, $n(p)$ denotes the total number of episodes the agent started at $p$ and adopted $\pi^\epsilon_{GO}$, $n_S(p)$ is the number of those episodes that satisfied the constraint under $\pi^\epsilon_{GO}$, $n_F(p)$ is the number of episodes that violated the constraint under $\pi^\epsilon_{GO}$, i.e., $n(p)=n_F(p)+ n_S(p)$. The value of $z$ is determined by the desired confidence level (e.g., 99\% confidence level corresponds to a $z$ value of 2.58), i.e, the probability that [$Pr_{low}(p)$, $Pr_{up}(p)$] contains $Pr_{GO}(p)$. Accordingly, we can select a high value of $z$ to ensure that $Pr_{GO}(p) \geq Pr_{low}(p)$ with high confidence. We update the lower bound $Pr_{low}(p)$ at the end of each episode. The resulting $Pr_{low}(p)$ is then used to update the switching probability $Pr_{switch}(p)$ of the initial state $p$.
 
 If $Pr_{low}(p)$ is less than the desired threshold $Pr_{des}$ (indicating a risk of violating the constraint), the switching probability should be set to 1 to ensure that the algorithm always employs $\pi_{GO}^{\epsilon}$ when starting at the initial state $p$. If $Pr_{low}(p)$ exceeds $Pr_{des}$, indicating a high likelihood of constraint satisfaction by $\pi_{GO}^{\epsilon}$ policy, the switching probability can be set lower than 1. This adjustment allows for executing pure RL with a non-zero probability to enhance reward maximization. The estimation of the satisfaction probability becomes more accurate as the number of episodes increases, and the algorithm reduces the switching probability as needed to achieve reward maximization. 
 
 Note that in any episode starting at $p$, the probability of satisfying the constraint is lower bounded by $Pr_{switch}(p)Pr_{GO}(p)$, i.e., the probability of choosing $\pi_{GO}^{\epsilon}$ in that episode times the probability of satisfying the constraint from that initial state via $\pi_{GO}^{\epsilon}$. Accordingly, we propose to update the switching probability as
 \begin{equation}
 \label{switching}
    Pr_{switch}(p) = min\left(1 , \frac{Pr_{des}}{Pr_{low}(p)}\right),
 \end{equation}
which ensures that the product, $Pr_{switch}(p)Pr_{GO}(p)$, a lower bound on the probability of satisfying the constraint starting at $p$, is at least $Pr_{des}$ as long as 1) $Pr_{GO}(p) \geq Pr_{des}$, and 2)  $Pr_{GO}(p) \geq Pr_{low}(p)$, which can be achieved with very high confidence via a proper selection of $z$  in \eqref{lower}.
 

At the beginning of each episode,  the agent decides whether to adhere to the $\pi_{GO}^{\epsilon}$ policy for constraint satisfaction or to employ RL for maximizing rewards. This decision is based on the calculated switching probability $Pr_{switch}(\cdot)$. The design of $Pr_{switch}$ in \eqref{switching} ensures that: 1) the agent exclusively follows the $\pi_{GO}^{\epsilon}$ policy when the confidence lower bound $Pr_{low}$ is lower than the desired threshold $Pr_{des}$; 2) the agent is allowed to engage in  RL for reward maximization when $Pr_{low}$ exceeds $Pr_{des}$ (as presented in Alg.~\ref{alg:learning}). 

\begin{algorithm2e}[!htbp] \small
    \DontPrintSemicolon
    \SetKwInOut{Input}{Input}
    \SetKwInOut{Output}{Output}
    \Input{
         product MDP $\mathcal{P} = (S_{P},P_{init},A,\Delta_{P}, R_{P}, F_{P})$; initial MDP state $s_{init}$; $\pi_{GO}^{\epsilon}$ policy; time bound of $\phi$, i.e., $T$}
    \Output{$\pi:S_{P} \rightarrow A$;\,\,$Pr_{switch}(\cdot)$}
    \SetKwFunction{FMain}{update}
    \SetKwProg{Pn}{Function}{:}{\KwRet}
    \textbf{Initialization:} 
    $n(p) \leftarrow 0, n_S(p) \leftarrow 0, n_F(p) \leftarrow 0$ for all $p \in P_{init}$\\
     \textbf{Initialization:} 
    $p \leftarrow find\,\, \bar{p}\in P_{init}\,\, s.t. \,\,mdp\_state(\bar{p})=s_{init}$\\
    \For{$j = 0:N_{episode}-1$}{
        $p_{0} \leftarrow p$\\
         \lIf{$n(p_0) < N_{sample}$ or random() $< Pr_{switch}(p_{0})$ }{$flag_{RL}$ $\leftarrow 0$} 
        \lElse{$flag_{RL}$ $\leftarrow$ 1}
        \color{black}
        \For{$t = 0:T-1$}{
            \uIf{constraint not satisfied \textbf{and} $flag_{RL}$ = $0$ }{
                Action $a$ $\leftarrow$ $\pi_{GO}^{\epsilon}(p)$\\
                Take action $a$, observe the next state 
            }
            \ElseIf{constraint satisfied  \textbf{or} $flag_{RL} = 1$}{
               Update $\pi$ via a selected RL algorithm }}
            \If{constraint satisfied}{
                $\FuncSty{update(}p_{0}, n(p_{0}),n_S(p_{0}),n_F(p_{0}),`success\text{'}\FuncSty{)}$
                } 
            \Else{
                $\FuncSty{update(}p_{0}, n(p_{0}),n_S(p_{0}),n_F(p_{0}),`failure\text{'}\FuncSty{)}$
            }  
           }
    \Pn{\FMain{$p_{0}$, $n(p_{0})$, $n_S(p_{0})$, $n_F(p_{0})$, $result$}}{ 
    $n(p_{0}) \leftarrow n(p_{0}) + 1$\\
    $n_S(p_{0}) \leftarrow n_S(p_{0}) + 1 $ \textbf{if} $result$ = `success', \textbf{else} $n_F(p_{0}) \leftarrow n_F(p_{0}) + 1$\\
    $Pr_{low}(p_{0}) \leftarrow$ equation (\ref{lower})\\
    $Pr_{switch}(p_{0}) \leftarrow$ equation (\ref{switching})
    }
    \caption{Switching-based RL}
     \label{alg:learning}
    \end{algorithm2e}  

Algorithm \ref{alg:learning} begins by initializing the numbers of trials, successes, and failures for every initial state in $ P_{init}$ (line 1). Line 2 sets the initial product MDP state. Before each episode, the algorithm records the initial state of the current episode (line 4) and determines whether to follow the $\pi_{GO}^{\epsilon}$ policy or to adopt RL (lines 5-6).  
In line 5, the condition $n(p_0) < N_{sample}$ is included to ensure enough samples have been collected for accurate estimation of the confidence lower bound. In situations where $\pi_{GO}^{\epsilon}$ policy is selected but the constraint has not yet been satisfied, the agent will take actions from the $\pi_{GO}^{\epsilon}$ policy (lines 8-10). Conversely, if the constraint has been satisfied or RL is selected, the agent uses a chosen RL algorithm to update the reward maximization policy $\pi$ (line 12). Example RL algorithms that can be used in line 12, such as Tabular Q-learning (\cite{watkins1992}) and Deep Q-learning (\cite{mnih2013playing}), are presented in Algs. 4 and 5. 
At the end of each episode, the algorithm will check if the BTL constraint is satisfied, and update the numbers of trials, successes, failures, and the switching probability for the initial state $p_0$ (lines 13-16), using the function \funcstyle{update()}. 


\begin{mytheorem} \label{theorem2}
    Given a BTL constraint $\phi$ with a desired probability threshold $Pr_{des}$ and a product MDP $\mathcal{P} = (S_{P},P_{init},A,\Delta_{P}, R_{P}, F_{P})$, if $Pr(p_0\xrightarrow{T} F_{P}; \pi_{GO}^{\epsilon}) \geq Pr_{des}$ \footnote{In practice, the proposed lower bounds, $lb^r$ or $lb^c$, can be used to verify this inequality as shown in Corollaries \ref{corollary_1} and \ref{corollary_2}. } for every initial state $p_0 \in P_{init}$, then  Alg.~\ref{alg:learning} guarantees that the probability of satisfying $\phi$ in each episode is at least $Pr_{des}$ with high confidence\footnote{Theorem~\ref{theorem2} does not claim that the probability of satisfaction is always greater than or equal to $Pr_{des}$. Instead, we ensure this probabilistic satisfaction guarantee with high confidence. This is because $Pr_{low}$ was estimated using the Wilson score method, which means that $Pr_{GO}(p_0) \geq Pr_{low}(p_0)$ (needed in \eqref{t3eq} in the proof) holds true with a high confidence level depending on the chosen parameter $z$ in \eqref{lower}.}. 
\end{mytheorem} 

\begin{proof} For each episode starting from an initial state $p_0$, the agent selects a policy to follow in that episode according to the switching probability $Pr_{switch}(p_0) = min\left(1 , \frac{Pr_{des}}{Pr_{low}(p_0)}\right)$.
 
 \vspace{1mm}
 \textbf{Case 1:} If $Pr_{des} \geq Pr_{low}(p_0)$, then $Pr_{switch} = 1$. The agent will adopt the $\pi_{GO}^{\epsilon}$ policy with probability 1. Accordingly, if $Pr(p_0\xrightarrow{T} F_{P}; \pi_{GO}^{\epsilon}) \geq Pr_{des}$, then the probability of satisfying $\phi$ in such an episode is at least $Pr_{des}$.

 \vspace{1mm}
 \textbf{Case 2:} If $Pr_{des} < Pr_{low}(p_0)$, then $Pr_{switch} = \frac{Pr_{des}}{Pr_{low}(p_0)} < 1$. The agent adopts $\pi_{GO}^{\epsilon}$ policy with probability $Pr_{switch}$ and RL with probability $1-Pr_{switch}$. 
 Starting from any initial state $p_0$, let $Pr_{sat}(p_0)$ represent the overall probability of satisfying the constraint in this episode, with $Pr_{GO}(p_0)$ denoting the satisfaction probability under $\pi_{GO}^{\epsilon}$ policy, and $Pr_{RL}(p_0)$ denoting the satisfaction probability under RL. Then,
 \begin{equation}
 \label{t3eq}
 \begin{split}
     Pr_{sat}(p_0) & = Pr_{GO}(p_0)\cdot Pr_{switch} + Pr_{RL}(p_0)\cdot (1-Pr_{switch}) \\
    & \geq Pr_{GO}(p_0)\cdot Pr_{switch}\\
    & \geq Pr_{low}(p_0)\cdot Pr_{switch} = Pr_{des}.
  \end{split}
 \end{equation}
\end{proof}

 


\section{Simulation Results}
    We present some case studies to validate the proposed algorithm and compare it with  \cite{aksaray2021probabilistically}, which serves as a baseline.
    The simulation results are implemented on Python 3.10 on a PC with an Intel i7-10700K CPU at 3.80 GHz processor and 32.0 GB RAM. We consider a robot operating on an 8×8 grid. The robot's action set is $A = \{N,NE,E,SE,S,SW,W,NW,Stay\}$, and the possible transitions under each action are shown in Fig.~\ref{fig:transition}. Action ``Stay" results in staying at the current position with probability 1. Any other action leads to the intended transition (blue) with a probability of 90\% and unintended transitions (yellow) with 10\%. 
    While this transition model is unknown to the robot, a conservative transition uncertainty $\epsilon \geq 0.1$ is provided, representing the robot's partial knowledge of the actual transition dynamics ($\epsilon=0.1$ is the actual transition uncertainty).

    \begin{figure}[htb!]
    \begin{center}
    \resizebox*{0.3\textwidth}{!}{\includegraphics{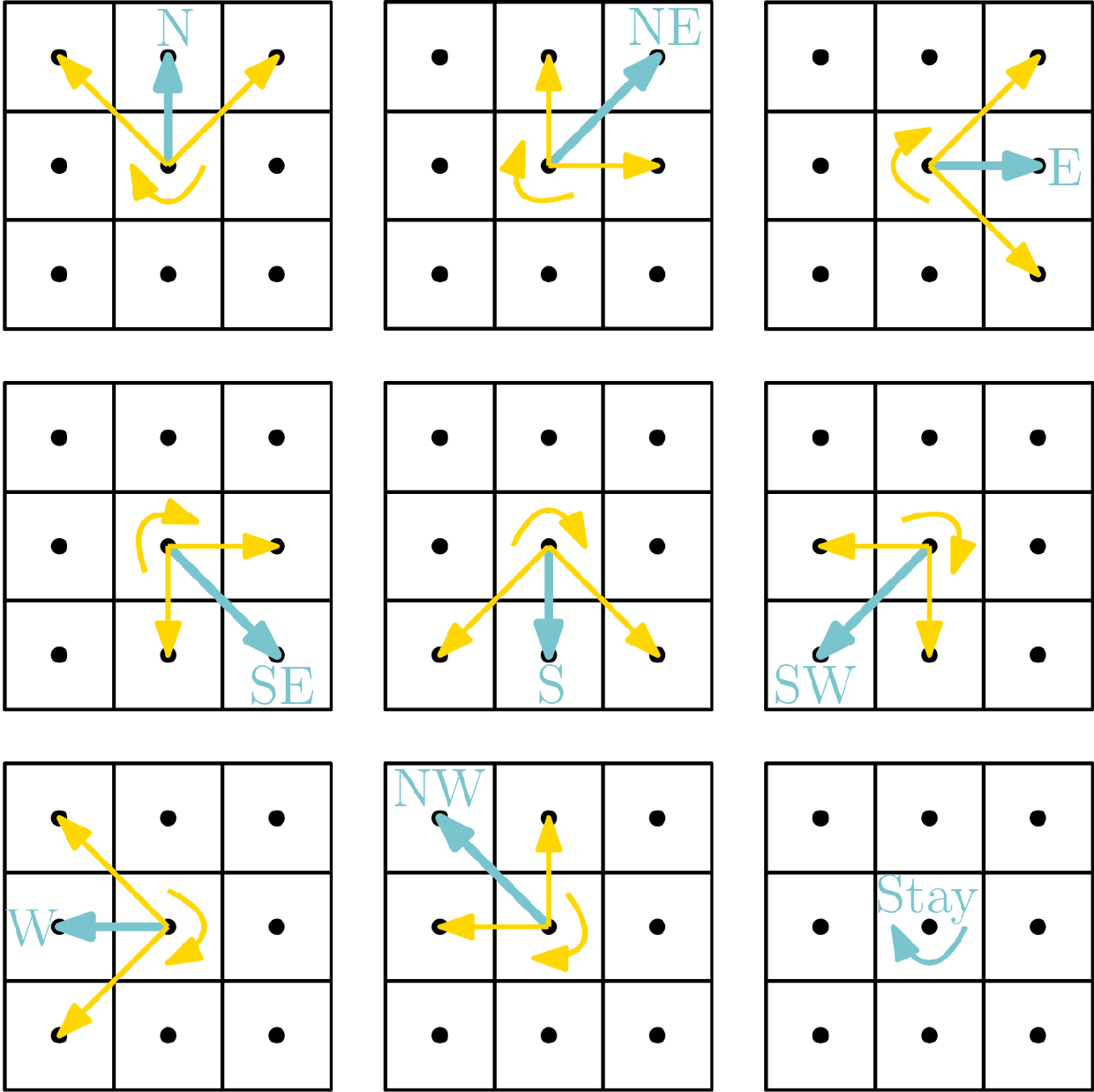}}%
    \caption{Transitions (intended - blue, unintended - yellow) under each action.}
    \label{fig:transition}  
    \end{center}
    \end{figure}

   We consider a scenario where the robot periodically performs a pickup and delivery task while monitoring high reward regions in the environment. In Fig.~\ref{fig:case1}, the light gray, dark gray, and all other cells yield a reward of 1, 10, and 0, respectively. The pickup and delivery task is formalized using a TWTL formula: $[H^1P]^{[0,20]}\cdot([H^1D_1]^{[0,20]} \lor [H_1D_2]^{[0,20]})\cdot[H^1Base]^{[0,20]}$, which specifies that the robot must ``reach the pickup location $P$ and stay there for 1 time step within the first 20 time steps, then immediately reach one of the delivery locations, $D_1$ or $D_2$, and stay there for 1 time step within the next 20 time steps; afterward, return to the Base and stay for 1 time step, within 20 steps." Based on the time bound of the formula, each episode's length is set at 62 time steps.

    \textbf{Case 1.} 
    We illustrate sample trajectories in Fig.~\ref{fig:case1} using the policy learned by the algorithm in \cite{aksaray2021probabilistically} (`baseline') and the proposed algorithm (with tabular-Q learning) after training for 40,000 episodes. The proposed algorithm switches between two different behaviors based on the selected mode while the baseline finds a single behavior that satisfies the constraint and maximizes the reward. Both algorithms not only satisfy the TWTL constraint with a desired probability but also effectively explore the high-reward regions. 

\begin{figure*}[htp]
    \centering
    \subfigure[]{%
      \includegraphics[clip,width=0.25\textwidth]{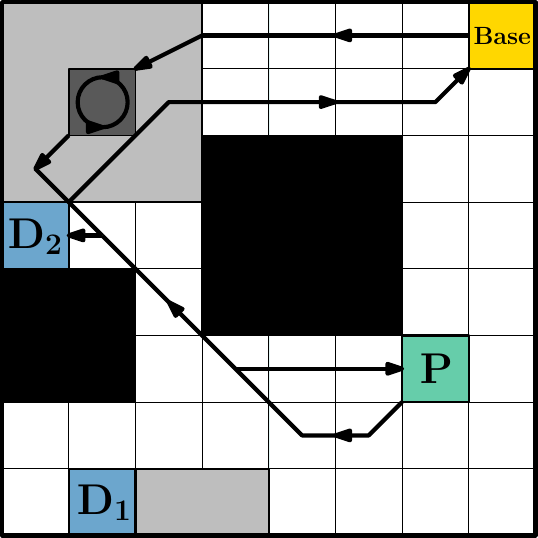}%
    }
    \hspace{1cm}
    \subfigure[]{%
      \includegraphics[clip,width=0.25\textwidth]{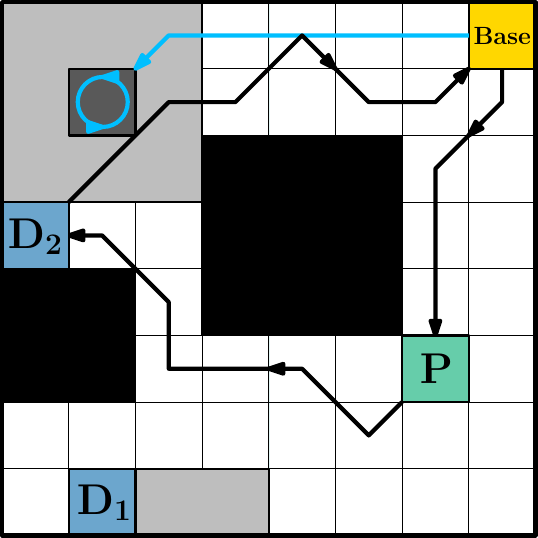}%
    }
    
    \caption{An environment where yellow, green, blue, and black cells are respectively the base station, the pick-up region, the delivery regions, and the obstacles. The gray cells are reward regions (darker shades - higher reward). The arrows denote illustrative trajectories that are obtained by applying policies learned by: (a)  \cite{aksaray2021probabilistically}, (b) the proposed algorithm (blue: reward maximization policy $\pi$, black: $\pi_{GO}^{\epsilon}$ policy).}
    \label{fig:case1}
\end{figure*}

In Cases 2 and 3, we consider a fixed number of episodes ($N_{episode}=1000$) and diminishing $\epsilon$-greedy policy in RL algorithms (with $\epsilon_{init}=0.7$ and $\epsilon_{final}=0.0001$). In this way, we compare the performances of both algorithms under fixed learning episodes and exploration/exploitation behavior. The learning rate and discount factor are set to 0.1 and 0.95, respectively. We set the $z$ score to 2.58 to ensure the probabilistic constraint satisfaction with high confidence. Each algorithm was run through 10 independent training sessions. For each run, the rewards and satisfaction rates were smoothed using a moving window average. The solid lines represent the average reward and satisfaction rate at each episode, calculated as the mean of the moving window averages from all 10 runs. The upper and lower bounds of the shaded areas indicate the maximum and minimum moving window averages over the 10 runs at each episode.

\textbf{Case 2.} We evaluated the baseline and the proposed algorithm (with Tabular Q-learning and Deep Q-learning)  
under varying $Pr_{des}$. As shown in Fig.~\ref{fig:case2}, the baseline tends to be over-cautious and enforces a higher satisfaction rate in all cases, while our proposed algorithm effectively balances constraint satisfaction with reward maximization, with the satisfaction rate adaptively aligned with the desired threshold. As $Pr_{des}$ increases, we notice a decrease in the collected rewards in both algorithms, due to a more restrictive constraint.

\begin{figure*}[htp]
\centering
\subfigure[$Pr_{des}=0.9, \epsilon = 0.2$]{%
  \includegraphics[clip,width=0.48\textwidth]{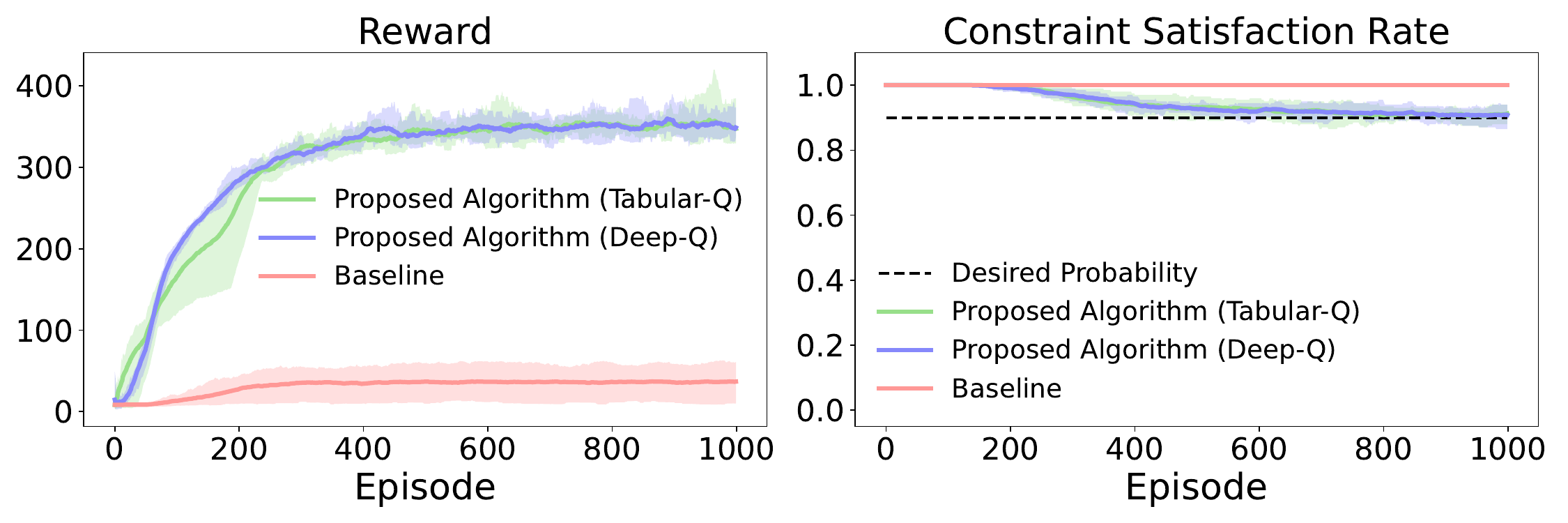}%
}
\subfigure[$Pr_{des}=0.7, \epsilon = 0.2$]{%
  \includegraphics[clip,width=0.48\textwidth]{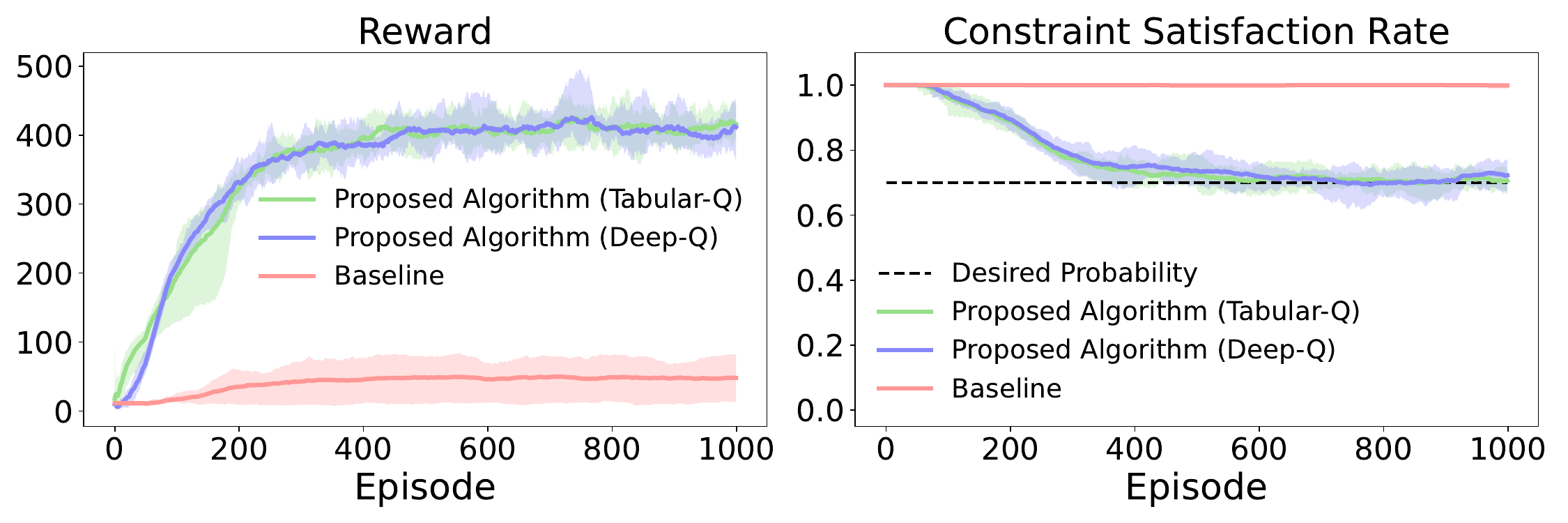}%
}

\caption{Reward and constraint satisfaction rate under various desired probabilities $Pr_{des}$.}
\label{fig:case2}
\end{figure*}

 \textbf{Case 3.} This case study investigates the impact of the parameter $\epsilon$ on the performance of both algorithms. 
 As in Fig.~\ref{fig:case3}, the baseline's performance is significantly affected by $\epsilon$; a higher $\epsilon$ leads to reduced reward collection and an increased satisfaction rate. This difference arises because the baseline uses $\epsilon$ to compute the lower bound of satisfaction and prune actions accordingly, and thus a larger $\epsilon$ will result in a more restricted action set. In contrast, the proposed algorithm is designed independently of $\epsilon$ by adapting the switching behavior based on collected data, rather than relying on the overestimated $\epsilon$. Consequently, the performance of the proposed algorithm remains largely unaffected by the change of $\epsilon$ in the experiments.

 \begin{figure*}[htp]
\centering
\subfigure[$Pr_{des}=0.6, \epsilon = 0.2$\label{fig:case3a}]{%
  \includegraphics[clip,width=0.48\textwidth]{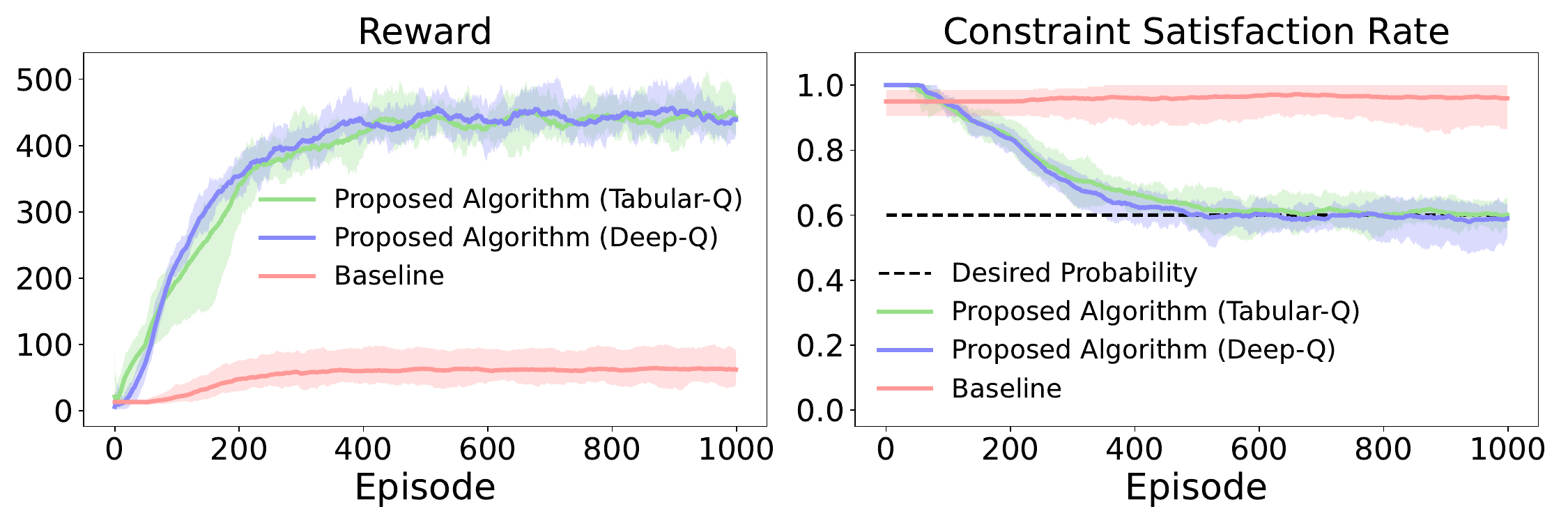}%
}
\subfigure[$Pr_{des}=0.6, \epsilon = 0.1$\label{fig:case3b}]{%
  \includegraphics[clip,width=0.48\textwidth]{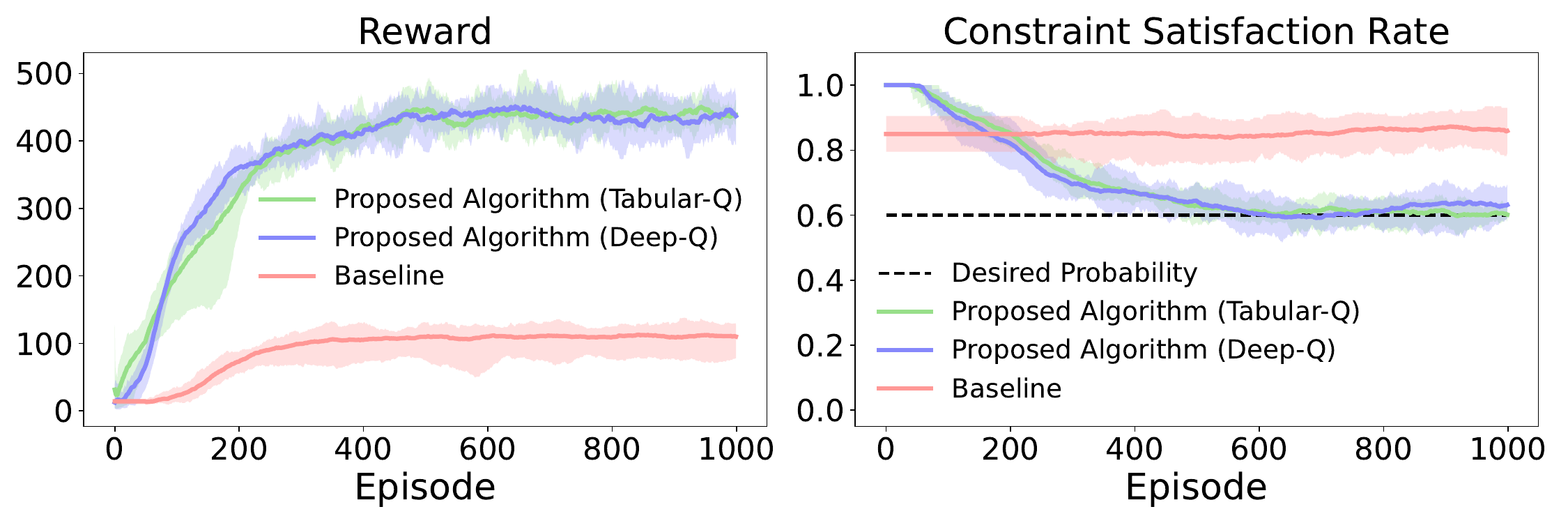}%
}

\caption{Reward and constraint satisfaction rate under different uncertainties $\epsilon$.}
\label{fig:case3}
\end{figure*}

The proposed algorithm consistently outperforms the baseline in maximizing rewards for two main reasons: 1) The proposed solution separates reward maximization from constraint satisfaction, allowing it to learn the reward-maximizing policy on a significantly smaller MDP, whereas the baseline operates on a much larger time-product MDP; 2) While the baseline tends to over-satisfy the TWTL constraint, the proposed solution adaptively aligns with the desired probability by adjusting the switching probability, providing more freedom to explore and maximize the rewards. In general, the proposed algorithm is also more efficient, requiring significantly fewer training episodes to converge compared to the baseline.

  \textbf{Case 4.} We compare the closed-form solution \eqref{lower_bound} with the recursive one (Alg. \ref{alg:recursive}) in terms of their ability to evaluate the lower bound and their computation efficiency. 
In Table 1, we present the computed lower bounds for selected product MDP states $p$ at a time step $k=17$ for a TWTL task $[H^1 P]^{[0,8]} \cdot [H^1 D_1]^{[0,8]}$ (``visiting $P$ within 8 time steps and holding 1 time step at P, after which visiting $D_1$ within 8 time steps and holding 1 time step at $D_1$"). The results indicate that the recursive solution consistently generates higher (less conservative) lower bounds than the closed-form solution.

 \begin{table}[h!]
 \centering
\begin{adjustbox}{width=0.6\columnwidth}
\begin{tabular}{||c c c c ||}
 \hline
 $k$ & $p$ & $lb^c[k][p]$ by (\ref{lower_bound}) & $lb^r[k][p]$ by Alg. \ref{alg:recursive} \\ [0.5ex] 
 \hline\hline
 \multirow{2}{*}{$k=17$} &
 \multicolumn{1}{l}{($P$, $q_0$)} &
 \cellcolor{gray!20}{0.814} & 
 \cellcolor{gray!45}{0.988} \\ &
 \multicolumn{1}{l}{($Base$, $q_0$)} & \cellcolor{gray!20}{0.359} & \cellcolor{gray!45}{0.798} 
 \\\hline
\end{tabular}
\end{adjustbox}
\caption{Lower bounds for the task 
$[H^1 P]^{[0,8]} \cdot [H^1 D_1]^{[0,8]}$}
\label{table:1}
\end{table}

In Table 2, we analyze the computation time for the closed-form and recursive solutions under various TWTL tasks. 
\textbf{Case 4a:} We consider a TWTL task of the form $[H^1P_1]^{[0,\,t_1]}\cdot[H^1P_2]^{[0,\,t_2]} \cdots$. By incrementally adding subtasks and adjusting their durations, while keeping the total task duration T constant, we increase the number of states in the product MDP and, consequently, the count of state-time pairs $(p,k)$. The results in Table 2 (Case 4a) reveal that the computation time of the recursive algorithm increases with the number of $(p,k)$ pairs while the closed-form solution is not affected. This result aligns with the expectation that the recursive algorithm's computational load increases due to the iterative solving of optimization problems for each $(p,k)$ pair. \textbf{Case 4b:} We consider a TWTL task of the form $[H^1P]^{[0,t_1]}\cdot([H^1D_1]^{[0,t_2]} |[H_1D_2]^{[0,t_2]})\cdot[H^1Base]^{[0,t_3]}$. While maintaining a fixed number of subtasks, we vary the duration of each subtask to alternate the total task duration $T$. As shown in Table 2, while the computation time for the closed-form solution increases significantly with the task duration $T$, it remains more computationally efficient than the recursive solution in either case, as expected. 

\begin{table*}
  \centering
    \resizebox{0.9\textwidth}{!}{%
    \begin{tabular}
            {|c|c|c|c||c|c|c|c|}
            \hline
            \multicolumn{4}{|c||}{\textbf{Case 4a: Fixed Horizon T and Varying \# of Subtasks} }  & \multicolumn{4}{c|}{\textbf{Case 4b: Varying Horizon T and Fixed \# of Subtasks}} \\ \hline
             \textbf{T} & \textbf{\# of (p,k) Pairs} & \textbf{Time for} $lb^c$ & \textbf{Time for} $lb^r$ & \textbf{T} & \textbf{\# of (p,k) Pairs} & \textbf{Time for} $lb^c$ & \textbf{Time for} $lb^r$ \\
            \hline
            62 & 7936 & 0.1029 & 7.68 & 32 & 8192 & 0.0159 & 8.711 \\
            \hline
             62 & 11904  & 0.1058 & 15.42  & 47 & 12032  & 0.0435 & 15.27 \\
            \hline
            62 & 15872  & 0.1019 & 20.78 & 62 & 15872  & 0.0891 & 21.14 \\
            \hline
            62 & 19840  & 0.1075 & 28.47 & 77 & 19712  & 0.1584 & 26.43 \\
            \hline
            62 & 23808  & 0.1029 & 35.35 & 92 & 23552  & 0.2531  & 32.57 \\
            \hline
    \end{tabular}
    }
    \captionof{table}{Computation Time of the Closed-form Solution and Recursive Algorithm}
    \label{tab:example}
\end{table*}

\section{Conclusion}
We proposed a switching-based algorithm for learning policies to optimize a reward function while ensuring the satisfaction of the BTL constraint with a probability greater than a desired threshold throughout the learning process. Our approach uniquely combines a stationary policy for ensuring constraint satisfaction and an RL policy for reward maximization. Utilizing the Wilson score method, we effectively estimate the satisfaction rate's confidence interval, thereby adaptively adjusting the switching probability between the two policies. This method achieves a desired trade-off between constraint satisfaction and reward collection. Simulation results have demonstrated the algorithm's efficacy, showing improved performance over existing methods. 

\bibliography{TRO}
\appendix
\section*{Appendix}
\noindent\textbf{Proof of Theorem 1.}
For any state $p$ whose distance-to-$F_{P}$ satisfies $0 <d< \infty$, there exists a shortest path on the product MDP from $p$ to $F_{P}$ consisting of only intended transitions, i.e., transitions that occur with probability at least $1-\epsilon$ under their respective actions. By design, the policy $\pi^{\epsilon}_{GO}$ selects actions that steer the system along such a shortest path. Under $\pi^{\epsilon}_{GO}$, \textit{intended transitions} reduce the distance to $F_{P}$ by one with a probability of at least $1-\epsilon$, while 
\textit{unintended transitions} can increase the distance to $F_{P}$ by at most $\delta_{max}$, as stated in Assumption \ref{assumption_1}. Accordingly, we derive a lower bound on the overall satisfaction probability by analyzing the probability of reaching $F_{P}$ in the remaining $k$ time steps under a worst-case scenario, where every transition under $\pi^{\epsilon}_{GO}$ is assumed to decrease the distance to $F_{P}$ by 1 with probability $1-\epsilon$ or increase it by $\delta_{max}$ with probability $\epsilon$. Accordingly, given $k$ remaining time steps, we can represent all possible changes in the distance to $F_{P}$ as a Bernoulli process with $k$ trials, denoted as $(X_1,\ldots, X_k)$. Here, the random variable $X_i$ takes the value 1 with a probability of $1-\epsilon$, and $-\delta_{max}$ with a probability of $\epsilon$. Since the distance from state $p$ to $F_{P}$ is $d$, all outcomes that reach $F_{P}$ within $k$ time steps can be expressed as the union of $k$ distinct sets  $T = T_1 \cup T_2 \cup \ldots \cup T_k$, where each $T_m$ consists of outcomes $(X_1,\ldots, X_k)$ that reach $F_{P}$ in $m$ steps, i.e., 
\begin{equation}\label{def_Tm}
    T_m=\{(X_1, \ldots, X_k)| \sum_{i=1}^{m} X_i = d, \hspace{0.2em} \text{and}\hspace{0.2em}  \sum_{i=1}^{m'} X_i < d, \forall m' <m\}.
\end{equation}
  Then, the lower bound $lb^c[k][p]$ can be derived by computing the probability of $T$, denoted by $P(T)$. Note that \eqref{def_Tm} implies
\begin{equation}\label{eq:k_step_property}
   T_f \cap T_g = \emptyset, \forall f \neq g \in \{1,2, \hdots, k\}.
\end{equation}
Hence,
\begin{equation}
\begin{split}
     P(T) & = P(T_1 \cup T_2 \cup \ldots \cup T_k) = P(T_1) + P(T_2) + \ldots + P(T_k).
\end{split}
\end{equation}

Furthermore, each $P(T_m)$ can be expressed as

\begin{equation}\label{eq:ptk_1}
\begin{split}
     P(T_m) 
     & \stackrel{\circled{1}}{=} P(\,\,\sum_{i=1}^{m} X_i = d \hspace{0.2em} \text{and}\hspace{0.2em}  \sum_{i=1}^{m'} X_i < d, \forall m'<m\,\,) \\
     & \stackrel{\circled{2}}{=} P(\,\sum_{i=1}^{m} X_i = d) - 
        P(\,\,\sum_{i=1}^{m} X_i = d \hspace{0.2em} \text{and}\hspace{0.2em} \exists m' <m \text{ s.t. } \sum_{i=1}^{m'} X_i \geq d \,)\\
     & \stackrel{\circled{3}}{=} P(\,\sum_{i=1}^{m} X_i = d) - 
        P(\,\,\sum_{i=1}^{m} X_i = d \hspace{0.2em} \text{and}\hspace{0.2em} (X_1,\ldots,X_k) \in \bigcup_{m'=1}^{m-1}T_{m'})\\
     & \stackrel{\circled{4}}{=} P(\,\sum_{i=1}^{m} X_i = d) - 
        \sum_{m'=1}^{m-1}P(\,\,\sum_{i=1}^{m} X_i = d \hspace{0.2em} \text{and}\hspace{0.2em} (X_1,\ldots,X_k) \in T_{m'})\\
     & \stackrel{\circled{5}}{=} P(\,\sum_{i=1}^{m} X_i = d) - 
       \sum_{m'=1}^{m-1}P(\sum_{i=m'+1}^{m} X_i = 0 \hspace{0.2em} \text{and}\hspace{0.2em} (X_1,\ldots,X_k) \in T_{m'})\\
     & \stackrel{\circled{6}}{=} P(\sum_{i=1}^{m} X_i = d) - 
        \sum_{m'=1}^{m-1}P(\sum_{i=m'+1}^{m} X_i = 0)P(T_{m'})\\
\end{split}
\end{equation} where each equality is obtained as follows:
\begin{itemize}
    \item[$\circled{1}$] directly follows from \eqref{def_Tm}.
    \item[$\circled{2}$] follows from the rule $P(A \cap B)=P(A)- P(A \cap \bar{B} )$.
    \item[$\circled{3}$] is due to the equality of the following sets $A$ and $B$: \newline 
    $A =\{ X \mid \sum_{i=1}^{m} X_i = d \hspace{0.2em} \text{and}\hspace{0.2em} \exists m'<m \text{ s.t. } \sum_{i=1}^{m'} X_i \geq d \}$, \newline  
    $B=\{X \mid \sum_{i=1}^{m} X_i = d\hspace{0.2em} \text{and}\hspace{0.2em}  \exists m'<m \text{ s.t. } X \in T_{m'}$\}$, \newline \text{which we show below by proving }$ $B\subseteq A$ and $A\subseteq B$. \newline $B \subseteq A$: Using \eqref{def_Tm},  $X \in B$ implies  $\exists m'<m \text{ s.t. } \sum_{i=1}^{m'} X_i = d $. Since such  $m'$ \text{ satisfy } $\sum_{i=1}^{m'} X_i \geq d $, $X \in A$. \newline
    $A \subseteq B$: For any $X=(X_1,\ldots,X_k) \in A$, let $ m'$ be the smallest integer $ \hspace{0.2em} s.t.\sum_{i=1}^{m'} X_i \geq d$. Since each $X_i \in \{1,-\delta_{max}\}$, the sum of $X_i$ can at most increase by 1 with each term, which implies $\sum_{i=1}^{m'-1} X_i = d-1$ and $\sum_{i=1}^{m'} X_i = d$. Accordingly, $X\in T_{m'}$, which implies $X\in B$.

\item[$\circled{4}$] is obtained by using \eqref{eq:k_step_property}. 
\item[$\circled{5}$] is obtained as follows: Using \eqref{def_Tm}, any $(X_1,\ldots,X_k) \in T_{m'}$ satisfies $\sum_{i=1}^{m'} X_i = d$. Hence the condition that $\sum_{i=1}^{m} X_i = d$ and $(X_1,\ldots,X_k) \in T_{m'}$ is equivalent to  $\sum_{i=m'+1}^{m} X_i = 0$ and $(X_1,\ldots,X_k) \in T_{m'}$.
\item[$\circled{6}$] follows from that the probability of $(X_1,\ldots,X_k) \in T_{m'}$ is independent of $X_{m'+1}, \hdots, X_k$, due to \eqref{def_Tm},  and $P(A \cap B) = P(A)P(B)$ when $A$ and $B$ are independent.
\end{itemize}
 
Next, we derive an expression for each $P(T_m)$ based on \eqref{eq:ptk_1}. We first compute $P(\sum_{i=1}^{m} X_i = d)$. Let  $p$ be the number of intended transitions and $q=m-p$ be the number of unintended transitions in $X_1,\ldots,X_m$.  Then, 
 \begin{equation}\label{eq:ptk_2}
   \sum_{i=1}^{m} X_i = d  
     \Leftrightarrow p-q\cdot \delta_{max}=d.
\end{equation}
Using \eqref{eq:ptk_2} and $q=m-p$, we obtain
 \begin{equation}\label{eq:ptk_3}
   \sum_{i=1}^{m} X_i = d  
     \Leftrightarrow p =  
            \frac{m\delta_{max}+d}{1+\delta_{max}}, q = \frac{m-d}{1+\delta_{max}}.
\end{equation}
Accordingly,
\begin{equation}
    \label{eq:ptk_4}  P(\sum_{i=1}^{m} X_i = d)  
       \begin{aligned}[t]
           & = C_{m}^{q}\epsilon^{q}(1-\epsilon)^{p}\\
           & = C_{m}^{\frac{m-d}{1+\delta_{max}}}\epsilon^{\frac{m-d}{1+\delta_{max}}}(1-\epsilon)^{\frac{m\delta_{max}+d}{1+\delta_{max}}},
        \end{aligned}
\end{equation}
 where $C_{m}^{n}$ is defined as 
 \begin{equation*}
      C_{m}^{n} =
  \begin{cases}
      0 & \text{if }n > m \hspace{0.4em} \text{or $n \notin \mathbb{Z^+}$}  \\
      \frac{m!}{n!(m-n)!} & \text{otherwise}
    \end{cases} 
 \end{equation*}
 
Similarly, to compute $P(\sum_{i=m'+1}^{m} X_i = 0)$, let there be $p$ intended actions and $q=m-m'-p$ unintended transitions in $(X_{m'+1},\ldots,X_m)$. Then, 
 \begin{equation}\label{eq:ptk_5}
 \sum_{i=m'+1}^{m} X_i = 0
     \Leftrightarrow p = \frac{(m-m')\delta_{max}}{1+\delta_{max}} , q = \frac{m-m'}{1+\delta_{max}}
     \end{equation}
Accordingly,
 \begin{equation}\label{eq:ptk_6}
 P(\sum_{i=m'+1}^{m} X_i = 0) 
       \begin{aligned}[t]
           & = C_{m-m'}^{q}\epsilon^{q}(1-\epsilon)^{p}\\
           & = C_{m-m'}^{\frac{m-m'}{1+\delta_{max}}}\epsilon^{\frac{m-m'}{1+\delta_{max}}}(1-\epsilon)^{\frac{(m-m')\delta_{max}}{1+\delta_{max}}}
        \end{aligned}
     \end{equation}

Plugging \eqref{eq:ptk_4} and \eqref{eq:ptk_6} into \eqref{eq:ptk_1}, we  to obtain an expression for each $P(T_m)$ as
\begin{equation}
    \label{eq:ptk_7}
    P(T_{m}) = \left[C_{m}^{\frac{m-d}{1+\delta_{max}}}\epsilon^{\frac{m-d}{1+\delta_{max}}}(1-\epsilon)^{\frac{m\delta{max}+d}{1+\delta_{max}}} - \sum_{m'=1}^{m-1}C_{m-m'}^{\frac{m-m'}{1+\delta_{max}}}\epsilon^{\frac{m-m'}{1+\delta_{max}}}(1-\epsilon)^{\frac{(m-m')\delta_{max}}{1+\delta_{max}}}P(T_{m'})\right],
\end{equation}
which then yields the expression for $lb^c[p][k] = \sum_{m=1}^{k} P(T_{m})$. 

\begin{algorithm2e}

    \DontPrintSemicolon
    \SetKwInOut{Input}{Input}
    \SetKwInOut{Output}{Output}
    \Input{
         Product MDP $\mathcal{P} = (S_{P},P_{init},A,\Delta_{P}, R_{P}, F_{P})$}
    \Output{Updated $\pi$ policy}
    \SetKwFunction{FMain}{update}
    \SetKwProg{Pn}{Function}{:}{\KwRet}
    Choose action $a$ from $\pi$ with $\epsilon$-greedy\\
    Take action $a$, observe the next state  $p'= (s',q')$ and reward $r$\\
    $Q(s,a) \leftarrow (1 - \alpha)Q(s,a)+\alpha(r+\gamma \max_{a'}Q(s', a'))$ \\
    $\pi(s) \leftarrow \argmax_a Q(s,a))$
       
    \caption{Tabular Q-learning}
     \label{alg:tabular-Q}
    \end{algorithm2e}
    
\vspace{-0.6cm}

\begin{algorithm2e}

    \DontPrintSemicolon
    \SetKwInOut{Input}{Input}
    \SetKwInOut{Output}{Output}
    \Input{
         product MDP $\mathcal{P} = (S_{P},P_{init},A,\Delta_{P}, R_{P}, F_{P})$}
    \Output{Updated $\pi$ policy}
    \SetKwFunction{FMain}{update}
    \SetKwProg{Pn}{Function}{:}{\KwRet}
    Choose action $a$ from $\pi$ with $\epsilon$-greedy\\
    Take action $a$, observe the next state  $p'= (s',q')$ and reward $r$\\
    Store transition ${(s,a,r,s')}$ in experience memory $D$\\
    Sample a random minibatch of transitions ${(s_j,a_j,r_j,s_{j+1})}$ from $D$ \\
    Set $y_j = r_j + \gamma \max\limits_a Q_{w-}(s_{j+1},a)$ \\
    Perform a gradient descent step on $(y_j - Q_w(s_j,a_j))^2$ \\
    Every C steps set $w_- = w$; \\
    $\pi(s) \leftarrow \argmax_a Q_w(s,a)$
       
    \caption{Deep Q-learning}
     \label{alg:DQN}
    \end{algorithm2e} 

\end{document}